\documentclass[twoside]{article}

\usepackage[accepted]{aistats2025}

\usepackage[utf8]{inputenc} 
\usepackage[T1]{fontenc}    
\usepackage{hyperref}       
\usepackage{url}            
\usepackage{booktabs}       
\usepackage{amsfonts}       
\usepackage{nicefrac}       
\usepackage{microtype}      
\usepackage{xcolor}         
\usepackage{natbib}

\usepackage{graphicx}
\usepackage{subfigure}
\usepackage{algorithm}
\usepackage{algorithmic}

\usepackage{multirow}
\usepackage{adjustbox}

\usepackage[resetlabels]{multibib}
\newcites{SM}{Supplementary Material References}
\usepackage{enumitem}

\usepackage{amsmath}
\usepackage{amssymb}
\usepackage{mathtools}
\usepackage{algorithmic}
\usepackage{booktabs}
\usepackage{tablefootnote}
\usepackage{color}
\usepackage{xspace}

\newcommand{\later}[1]{}

\newcommand{\name}{\textsc{RegVar}}
\newcommand{\namep}{\textsc{RegVar}\xspace}

\newcommand{\E}{\mathbb{E}}
\newcommand{\Var}{\mathrm{Var}}
\newcommand{\g}{\;|\;}

\usepackage{scalerel,stackengine,amsmath}

\newcommand{\D}{\mathcal{D}}

\newcommand{\objective}{\mathcal{L}_\theta}

\newcommand{\nnnl}{\nonumber \\}

\newcommand{\at}[1]{\biggr\rvert_{#1}}

\newcommand{\thetahat}{\hat{\theta}}
\newcommand{\dat}{\mathcal{D}}

\newcommand{\act}[1]{\texttt{\color{blue} //{#1}}}

\usepackage{amsmath}
\usepackage{amssymb}
\usepackage{mathtools}
\usepackage{amsthm}

\usepackage{fancyvrb}
\usepackage{xcolor}
\usepackage{mdframed}

\usepackage{caption}   
\usepackage{array}     

\theoremstyle{plain}
\newtheorem{theorem}{Theorem}[section]

\theoremstyle{definition}

\theoremstyle{remark}

\usepackage[textsize=tiny]{todonotes}

\begin{document}

\runningtitle{Variation Due to Regularization Tractably Recovers Bayesian Deep Learning}

\twocolumn[

\aistatstitle{Variation Due to Regularization Tractably\\Recovers Bayesian Deep Learning}

\aistatsauthor{James McInerney \And Nathan Kallus}

\aistatsaddress{ Netflix Research\\New York, NY, USA\\\texttt{jmcinerney@netflix.com} \And Netflix Research \& Cornell University\\New York, NY, USA\\\texttt{nkallus@netflix.com}} ]

\begin{abstract}
Uncertainty quantification in deep learning is crucial for safe and reliable decision-making in downstream tasks. Existing methods quantify uncertainty at the last layer or other approximations of the network which may miss some sources of uncertainty in the model. To address this gap, we propose an uncertainty quantification method for large networks based on \emph{variation due to regularization}. Essentially, predictions that are more (less) sensitive to the regularization of network parameters are less (more, respectively) certain. This principle can be implemented by deterministically tweaking the training loss during the fine-tuning phase and reflects confidence in the output as a function of all layers of the network. We show that regularization variation (\name) provides rigorous uncertainty estimates that, in the infinitesimal limit, exactly recover the Laplace approximation in Bayesian deep learning. We demonstrate its success in several deep learning architectures, showing it can scale tractably with the network size while maintaining or improving uncertainty quantification quality. Our experiments across multiple datasets show that \namep not only identifies uncertain predictions effectively but also provides insights into the stability of learned representations.
\end{abstract}

\section{INTRODUCTION}

\begin{figure*}[t!]
\vskip 0.2in
\begin{center}
\centerline{\includegraphics[width=1.18\columnwidth]{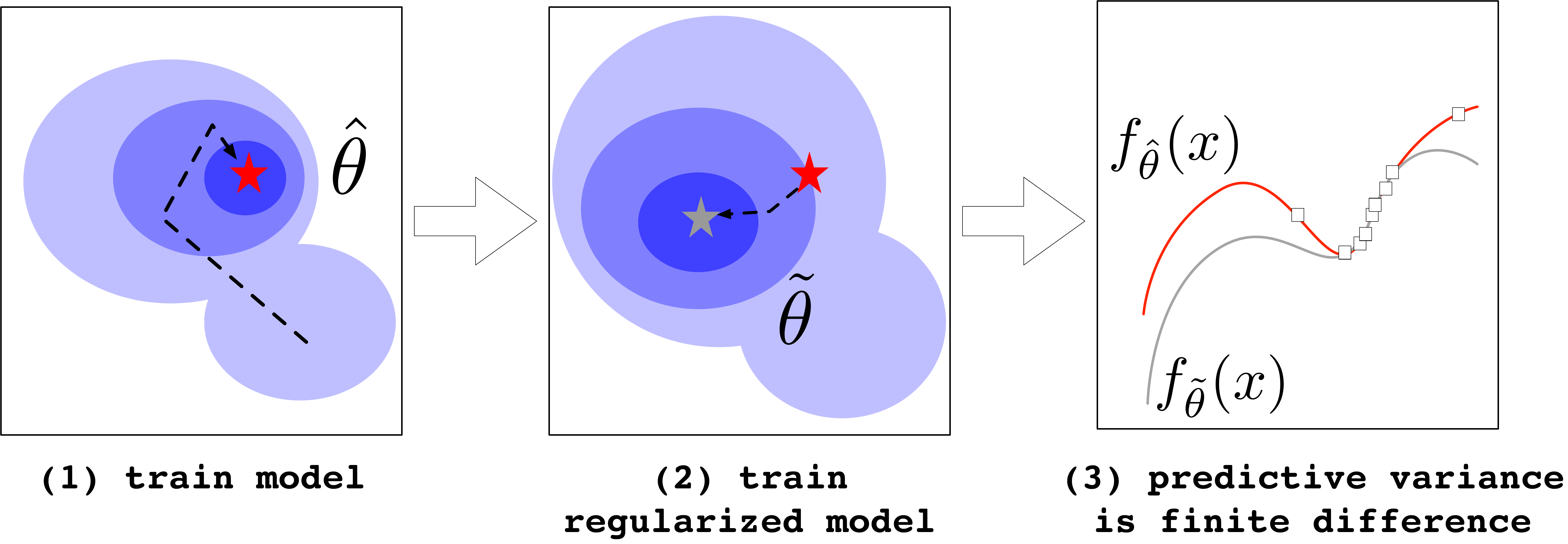}}
\caption{Summary of the regularization-based approach 
to uncertainty quantification that underlies \name. 
{\bf Step~1:} fit a predictive model by MAP; 
{\bf step~2:} fit another model using an infinitesimally regularized MAP;
{\bf step~3:} the difference between the predictions of these two models, normalized by the amount of regularization, is exactly the estimate of the predictive variance. 
}
\label{fig:infographic}
\end{center}
\end{figure*}

\definecolor{darkgrey}{rgb}{0.33, 0.33, 0.33}

\begin{figure*}[t]
\centering
\begin{mdframed}
{\footnotesize 
\begin{Verbatim}[commandchars=\\\{\}]
{\color{gray}// training step to update model parameters theta with training sample (x, y)}
{\color{gray}// n: training data size, lam: scalar hyperparameter}
f = model(x)
loss = cross_entropy(f, y) - log_prior(theta) / n + {\color{blue}lam * f.abs().mean() / n}
loss.backward()
optimizer.step()
...

{\color{gray}//inference time for input x_test}
{\color{blue}f_pred_var = (model(x_test) - f_pred_mean).abs() / lam}
\end{Verbatim}
}
\end{mdframed}
\caption{Example of fine-tuning a model using regularization variation (\name). 
The code required to extend existing fine-tuning to implement \namep is given in blue. 
Any auto-differentiation package is sufficient, here we use PyTorch~\citep{paszke2017automatic}.
}
\label{fig:implement}
\end{figure*}

Bayesian neural networks, where a prior is placed on model weights, offer an opportunity to understand uncertainty and improve model robustness~\citep{wilson2020case, papamarkou2024position}. However, exact computation of the model-weight posterior and predictive distributions is prohibitive given size and complexity. It is instead appealing to approximate it using the maximum \emph{a posteriori} (MAP) solution and the curvature thereat. However, even this curvature, characterized by a Hessian, can be challenging to compute.

The Laplace approximation of the posterior distribution uses 
the (negative) Hessian of the log joint distribution at the MAP~\citep{mackay1992bayesian}. 
For high dimensional $\theta$ -- common in deep neural networks -- 
these steps are a prohibitive computational bottleneck even for an approximate posterior.
One remedy is further approximation of the Hessian,
but this may ignore important directions of curvature. 

In this paper we propose using the principle of variation due to regularization (\name), which
directly approximates the uncertainty of the predicted mean 
in the Laplace approximation without explicitly computing and inverting the Hessian (while still assuming it exists). 
Instead, one additional point estimate is required, 
the \emph{prediction-regularized} MAP, 
derived from the network objective 
with a small amount of regularization added. 
Under the assumptions that the Laplace approximation is used, 
we find that the 
rescaled difference in network outputs 
given by the MAP and the 
prediction-regularized MAP 
recovers the Laplace variance.

Our contributions in this paper are as follows:
\begin{itemize}
	\item We develop a set of novel methods around regularization variation (\name) 
	and formalize it as the predictive variance of linearized Laplace. 
	\item In \name, variance is derived from the 
	change in the MAP prediction when adjusting the regularizer. 
	The potential appeal of \namep is the simplicity of the method (see Figure~\ref{fig:infographic}) and implementation (see Figure~\ref{fig:implement} for an example).
	\item We show that \namep scales up to large networks and compare it to several other popular Bayesian deep learning approaches. 
\end{itemize}

The rest of the paper is organized as follows. 
In Section~\ref{sec:related}, we discuss related work. 
In Section~\ref{sec:method}, 
we set up the key aspects of the Laplace approximation 
and connect it to regularization variation in Section~\ref{sec:hfl}. 
Empirical evaluation is explored in Section~\ref{sec:experiments} 
before discussing conclusions and future work in Section~\ref{sec:conclusions}.

\section{RELATED WORK}
\label{sec:related}

The Laplace approximation 
was first formulated in BDL 
for small-scale neural networks~\citep{mackay1992}.  
Since then, 
larger architectures 
have necessitated Hessian approximations, 
including Gauss-Newton~\citep{foresee1997gauss} 
and Kronecker factorization that 
places a block diagonal structure on the covariate 
in line with assumed independence between the layers of a network~\citep{martens2015optimizing}.

In recent years, 
with the aforementioned scalable approximations 
and renewed appreciation of the advantages of Laplace 
as a \emph{post-hoc} uncertainty method~\citep{daxberger2021laplace}, 
a number of new directions were addressed, 
including relating neural network loss functions to Gaussian process inference~\citep{khan2019approximate} 
and exploring the advantages of locally linearized Laplace in out-of-distribution evaluation~\citep{foong2019between}. 
Underfitting in Laplace \citep{lawrence2001variational} 
is mitigated in the GGN approximation 
by switching to the generalized linear model 
implied by the local linearization \citep{immer2021improving}. 
Further scaling with linearized Laplace 
applies neural tangent kernels~\citep{deng2022accelerated} 
and variational approximation~\citep{ortega2023variational}. 
For large models, 
Laplace can be applied to a low-rank adaptor~\cite{yang2024bayesian}.

Beyond Laplace-based methods, a number of approaches 
have been proposed for scalable Bayesian inference in neural networks. 
Given the breadth of this topic, 
we give a non-exhaustive account (see \citealp{papamarkou2024position} for further discussion). Variational inference approximates the posterior 
with a parameterized distribution, 
enabling efficient optimization 
via stochastic gradient descent~\citep{graves2011practical, blundell2015weight, shen2024variational}. 
Stochastic gradient Langevin dynamics and related approaches 
that take samples from the optimization trajectory of a neural network 
use gradient noise to approximate Bayesian updates 
without explicit posterior evaluation~\citep{welling2011bayesian, mandt2016variational, maddox2019simple}. 
Dropout-based uncertainty estimation 
interprets dropout training as an 
approximate Bayesian inference method~\citep{gal2016dropout}. 
Ensembles of independently trained networks 
use multiple predictions to quantify the uncertainty of black-box models~\citep{lakshminarayanan2017simple, osband2015bootstrapped}.

Owing to the ubiquity of the Hessian of the objective 
and its computational demands, 
approaches to avoiding its direct evaluation 
are developed in areas of research distinct of Laplace. 
In particular, Hessian-free optimization 
uses the conjugate gradient method to 
iteratively update a matrix-vector product 
as part of an inner loop in 
Newton's method second-order optimization \citep{martens2011learning, pearlmutter1994fast}. 
Finally, \cite{kallus2022implicit} develop a frequentist approach 
establishing that infinitesimal regularization of the total log likelihood 
yields a variance estimate that is asymptotically equivalent to the delta method. 
While avoiding full Fisher matrix computation, this approach requires evaluating a perturbed estimator per test point and output dimension, leading to computational costs that scale accordingly. 
It is also a frequentist perspective, in contrast to the Bayesian approach developed in this paper.

\section{BACKGROUND}

\label{sec:method}

Consider the maximum \emph{a posteriori} (MAP) solution 
of a Bayesian neural network $f_\theta$ 
with parameters~$\theta$ given data $\left\{ (x_i, y_i) \right\}_{i=1}^n$, 
\begin{flalign}
	\thetahat &\in \arg_\theta \max \objective \label{eq:dnn1}\\
	\mathrm{where\;} \objective &= \sum_{i=1}^n \log p(y_i \g f_\theta(x_i)) + \log p(\theta), \label{eq:dnn2}
\end{flalign}
for some observation likelihood $p(y \g \cdot)$ and prior $p(\theta)$. This is equivalent to optimizing a loss (e.g., squared loss for Gaussian likelihood or cross-entropy for categorical likelihood) regularized by $\log p(\theta)$ and means that computing the MAP of a Bayesian neural network is akin to fitting a vanilla neural network.

Our starting point is the 
learning of parameters 
$\theta$ of a deep neural network $f_\theta$ as per Eq.~\ref{eq:dnn1}~and~\ref{eq:dnn2}. 
While the representational power and accuracy of 
the network may be high, 
in many applications it is crucial to 
also quantify the uncertainty of predictions, 
such as in autonomous vehicles, healthcare, or recommendations~\citep{kendall2017uncertainties, leibig2017leveraging}. 

Bayesian deep learning (BDL) provides a framework 
for analyzing uncertainty in deep learning~\citep{papamarkou2024position}. 
The key components are 
the prior distribution $p(\theta)$ and likelihood function $p(y \g \theta, x)$ 
in a family of models indexed by $\theta$, 
and an i.i.d. dataset $\dat = \left\{ (x_i, y_i) \right\}_{i=1}^n$ 
assumed to be collected from a model in this family. 
The posterior distribution $p(\theta \g \dat)$ 
may be flexibly used as the density (or mass) 
for the expectation of any downstream prediction depending on $\theta$, 
e.g. expectation or variance of predicted mean, 
as part of a larger simulation involving fitted deep models, 
or interpreting the variance of weights for different layers of the network. 
The main evaluations of interest in BDL 
are the mean $\E_{p(\theta \g \dat)}[f_\theta(x)]$ and (co)variance $\Var_{p(\theta \g \dat)}[f_\theta(x)]$ 
of the prediction 
for a query vector (or set of query vectors) $x$. 
These will be the focus of our study in this paper.

BDL is doubly intractable. 
First, the normalizing factor in the Bayesian posterior is intractable 
for all but the most trivial networks as it 
entails a sum over high-dimensional $\theta$. 
Many techniques have 
been developed to approximate the posterior -- 
the most common being 
MAP estimation, variational inference, and Markov chain Monte Carlo --  
and these vary in their scalability and applicability 
to deep neural nets \citep{blundell2015weight, ritter2018scalable}. 
A key desideratum 
for approximate inference 
that is often overlooked 
is the extent to which it is compatible 
with existing deep learning frameworks and implementations,  
which may be supported by 
years of model selection, tuning, and vast training budgets. 
The MAP point estimate is the most practical in this sense, 
and is the basis of the Laplace approximation. 
Second, even with an approximate posterior to hand, 
the posterior predictive $p(y \g \dat, x)$ 
depends on $f_{\theta}$, 
meaning that the integral cannot be further simplified. 
Monte Carlo estimation using samples from the posterior  
is used to address this intractability.

\paragraph{Laplace Approximation} 
Let $q$ be the Laplace approximation of the posterior, defined as the second-order Taylor expansion around $\thetahat$,
\begin{flalign}
	\log q(\theta) = \log p(\thetahat \g \dat) - \frac{1}{2} (\theta - \thetahat)^\top P (\theta - \thetahat), \label{eq:laplace}
\end{flalign}
where $P := -\nabla \nabla_\theta \log p(\theta \g \dat)$.
The first derivative does not appear in Eq.~\ref{eq:laplace} because $\thetahat$ is defined as a local optimum of the posterior, where it exists. 

Noticing that Eq.~\ref{eq:laplace} takes the form of a quadratic in $\theta$, we see that $q(\theta) = \mathcal{N}(\thetahat, P^{-1})$. If precision matrix $P$ can be calculated and inverted, then the first intractability 
of approximating the posterior is addressed. In addition, the normal form of $q$ also 
provides a closed-form solution to the model evidence, supporting model and hyperparameter selection \citep{mackay1992}. 
We next discuss precision calculation in more detail. 

\paragraph{Precision Matrix}
Distribution $q(\theta)$ requires the covariance matrix which entails 
calculating and inverting the precision matrix $P$ 
decomposed as, 
\begin{flalign}
    P &= -\nabla \nabla_\theta \log p(\theta \mid \dat) \nonumber\\ 
    &= \underbrace{-\nabla \nabla_\theta \log p(\theta)}_{\text{Gaussian prior} \implies \text{scalar matrix}} \quad - \underbrace{\nabla \nabla_\theta \log p(\dat \mid \theta)}_{\text{Hessian}\; H}. \label{eq:precision}
\end{flalign}
The prior term in Eq.~\ref{eq:precision} usually takes a simple form, 
e.g. constant scalar for Gaussian prior, 
so our focus from this point will be on $H$, the Hessian of the \emph{likelihood}.

The Hessian of the likelihood term in Eq.~\ref{eq:precision} may be further decomposed as, 
\begin{flalign}
    H &= \underbrace{\nabla_f \log p(y \mid f_\theta(x))}_{\text{residual}} {\nabla \nabla_\theta f_\theta(x)} + \notag \\ & 
    \underbrace{\nabla \nabla_f \log p(y \mid f_\theta(x))}_{\text{observation precision}} 
    \quad \nabla_\theta f_\theta(x) \nabla_\theta f_\theta(x)^\top.
    \label{eq:hessian_lik}
\end{flalign}

In practice, the precision calculation of the Laplace approximation 
drops the first term -- the network Hessian -- in Eq.~\ref{eq:hessian_lik} 
due to computational constraints. 
This is known as  
generalized Gauss-Newton (GGN)  
and is proportional to the outer product of the network Jacobian 
(second term in Eq.~\ref{eq:hessian_lik}). 
When GGN 
is used in approximate inference,  
it assumes the network takes the form of a generalized linear model. 
In Laplace, this is equivalent to a local linearization of the network prediction around $\thetahat$,
	\begin{flalign}
		\tilde{f}_\theta(x) = f_{\thetahat}(x) + \nabla_\theta f_{\theta}(x)\at{\thetahat}^\top (\theta - \thetahat), \label{eq:local_linear}
	\end{flalign}
which experimentally results in more accurate predictions than the original network $f_{\thetahat}$ when used in the Monte Carlo average for evaluation~\citep{foong2019between, immer2021improving}. 
However, the Hessian, even after simplifying with GGN, 
is still unwieldy or prohibitive to compute, store, and invert 
for large networks with high dimensional $\theta$.

\section{REGULARIZATION VARIATION (\name)}
\label{sec:hfl}

Against this background, 
we target the same mathematical object as the Laplace approximation, 
requiring the Hessian to exist but 
avoiding the need to calculate or invert it. 
To start, notice that an immediate consequence of Eq.~\ref{eq:local_linear} 
is,
	\begin{flalign}
		\tilde{f}_\theta \sim \mathcal{N}(f_{\thetahat}(x), \nabla_\theta f_{\thetahat}(x)^\top P^{-1} \nabla_\theta f_{\thetahat}(x)), \label{eq:bdm}
	\end{flalign}
where $\nabla_\theta f_{\thetahat} := \nabla_\theta f_{\theta}\at{\thetahat}$ is used for more compact notation. 
This is an instance of the Bayesian delta method~\citep{wasserman2006all}, differing from the classic delta method by the inclusion of the 
prior term in $P$ in Eq.~\ref{eq:precision}. 

The posterior predictive depends on the form of the 
observation likelihood function. 
For regression tasks, Gaussian observation 
noise with fixed standard deviation 
$\sigma$ is the standard, 
resulting in posterior predictive equal to Eq.~\ref{eq:bdm} plus additional variance term $\sigma^2$. 
For classification tasks, 
the posterior predictive is non-analytical, 
requiring Monte Carlo averaging over posterior samples. 
Notwithstanding the particular form of the posterior predictive distribution, 
the \emph{epistemic uncertainty} of the network predictions 
-- meaning, the error due to lack of knowledge --
is represented 
by Eq.~\ref{eq:bdm} 
and has considerable value in active learning, experimental design, and exploration-exploitation. 

We set up the prediction-regularized joint distribution as equal to the joint distribution (Eq.~\ref{eq:dnn2}) with an additional term proportional to the prediction at query point $x$,
\begin{flalign}
	{\objective}^{(f(x),\lambda)} &= \sum_{i=1}^n \log p(y_i \g f_\theta(x_i)) + \log p(\theta) + \lambda f_\theta(x),
	\label{eq:pr_joint}
\end{flalign}
for some constant $\lambda \in \mathbb{R}$. 
The prediction-regularized MAP is defined as,
\begin{flalign}
	\thetahat^{(f(x),\lambda)} \in \arg_\theta \max {\objective}^{(f(x),\lambda)}. \label{eq:opt_theta_lambda}
\end{flalign}

\newcommand{\thetalam}{\hat{\theta}^{(x,\lambda)}}

\begin{theorem}
\label{thrm:hfl}
The derivative of the prediction under the prediction-regularized MAP w.r.t. $\lambda$ recovers the variance-covariance term in Eq.~\ref{eq:bdm},
\begin{flalign}
	\nabla_{\lambda} f_{\thetahat^{(f(x),\lambda)}}(x) \at{\lambda=0} = \nabla_\theta f_{\thetahat}(x)^\top P^{-1} \nabla_\theta f_{\thetahat}(x), \label{eq:hfl1}
\end{flalign}
assuming $\thetahat^{(f(x),\lambda)}$ and $P^{-1}$ exist, and that ${\objective}^{(f(x),\lambda)}$ is continuously differentiable w.r.t. $\theta$. 
\end{theorem}

\begin{proof} See Appendix~A. 
\end{proof}

A parallel frequentist approach is known as the implicit delta method \citep{kallus2022implicit}. 
The relationship in Eq.~\ref{eq:hfl1} can also be understood as an application of the implicit function theorem (IFT) by Cauchy, e.g. see \citep{lorraine2020optimizing} for a modern treatment of IFT. 

\begin{algorithm*}[t!]
   \caption{Regularization Variation (\name)}
   \label{alg:pt_hfl}
\begin{algorithmic}
   \STATE {\bfseries Input:} network $f_\theta$, fitted MAP $\hat{\theta}$, training data $\D$ of size $n$, evaluation inputs $\{ \hat{x}_{i=1}^m \} $, scalar $\lambda$, step size schedule $\gamma(\cdot)$
   \STATE Warm start $\hat{\theta}^{(\hat{r}, \lambda)} \gets \hat{\theta}$
   \STATE Initialize step counter $j=0$
   \REPEAT
   \STATE Sample training example $(x_i,y_i) \sim \D$
   \STATE Follow gradient $\hat{\theta}^{(\hat{r}, \lambda)} \gets \hat{\theta}^{(\hat{r}, \lambda)} + n \gamma(j) \nabla_\theta \mathcal{L}_i\;\;\;$ \act{ see Eq.~\ref{eq:pt_hfl_abs}} 
   \STATE $j \gets j + 1$
   \UNTIL{$\hat{\theta}^{(\hat{r}, \lambda)}$ converges}
   \STATE {\bfseries Return:} predictive variance function $\hat{\sigma}_f^2(x) := \frac{1}{\lambda} | f_{\hat{\theta}^{(\hat{r}, \lambda)}}(x) - f_{\hat{\theta}}(x) |$
\end{algorithmic}
\end{algorithm*}

A corollary of Theorem~\ref{thrm:hfl} 
is that detecting the local change in prediction 
as $\lambda$ is increased (or decreased) from 0 
yields the variance term of the linearized Laplace approximation in~Eq.~\ref{eq:bdm}. 
The local change, i.e. the LHS of Eq.~\ref{eq:hfl1}, may be calculated 
by finite differences 
for some suitably small~$\lambda$,
\begin{flalign}
	\frac{1}{\lambda} (f_{\thetahat^{(f(x),\lambda)}}(x) - f_{\thetahat}(x)) \underset{\lambda \to 0}{\rightarrow} \nabla_{\lambda} f_{\thetahat^{(f(x),\lambda)}}(x) \at{\lambda=0} .   \label{eq:fd_hfl}
\end{flalign}
It is natural to consider auto-differentiation 
to calculate the LHS of Eq.~\ref{eq:hfl1}. 
From the fact that the optimum $\hat{\theta}^{(f(x),\lambda)}$ also 
changes with $\lambda$ it can be seen that 
auto-differentiation must deal with complex operations on optima 
and is unlikely to lead to a computational advantage 
unless used in tandem with additional work-arounds.

Putting it together, the regularization variation (\name) approximation is,
	\begin{flalign}
		\tilde{f}_\theta \sim \mathcal{N} \left(f_{\thetahat}(x), \frac{1}{\lambda} (f_{\thetahat^{(f(x),\lambda)}}(x) - f_{\thetahat}(x)) \right), \label{eq:hfl}
	\end{flalign}
and differs from linearized Laplace (Eq.~\ref{eq:bdm}) by the lack of an explicit Hessian. 
It is apparent by this comparison  
that \namep trades the computational and storage 
expense of dealing with the Hessian for 
that of a point-wise approximation 
for each query $x$. 
This operation may be amortized into the network, 
which we discuss next.

\subsection{Amortized Regularizer}
\label{sec:method_finetuning}

\namep targets the same predictive variance-covariance as the full linearized Laplace approximation. However, when evaluating the uncertainty of many predictions, the requirement to optimize to the $f$-objective specific to each test point is a bottleneck. In order to scale \namep in the number of test points, we develop an amortized regularization approach that requires only a single model to represent the $f$-regularized parameters for a set of evaluation points. 

Amortization is formulated as follows. 
Find the regularization term $\hat{r}$ 
for a finite set of evaluation inputs ${\bf \hat{x}} = \{\hat{x}_{1:m}\}$ 
that minimizes the average squared L2 norm error of 
the corresponding $f$-regularized parameters, 
\begin{flalign}
	\hat{r} = \arg_{r} \min \sum_{i=1}^m \| \hat{\theta}^{(f(x_i), \lambda)} - \hat{\theta}^{(r, \lambda)}  \|_2^2. \label{eq:pt}
\end{flalign}
Eq.~\ref{eq:pt} facilitates the calculation of Eq.~\ref{eq:hfl} using $\hat{\theta}^{(\hat{r}, \lambda)}$ instead of $\hat{\theta}^{(f(x),\lambda)}$ for arbitrary $x$ in the domain of $f$. 

\begin{theorem}
\label{thrm:pt}
The optimal regularization term defined by Eq.~\ref{eq:pt} is $\hat{r} = \frac{1}{m} \sum_{i=1}^m f(\hat{x}_i)$ 
when $\hat{\theta}^{(f(x),\lambda)}$ and $P^{-1}$ exist and ${\objective}^{(f(x),\lambda)}$ is continuously differentiable w.r.t. $\theta$. 
\end{theorem}

\begin{proof} See Appendix~A. 
\end{proof}

Theorem~\ref{thrm:pt} implies that 
embedding the average network output 
in the training objective 
minimizes the L2 norm of the $f$-regularized objective. 
Furthermore, we can absorb 
a sign term into $\lambda$ 
that depends on the sign 
of the output function 
by defining,
\begin{flalign}
	\tilde{\lambda} = \lambda  (-1)^{\mathbb{I}[f_\theta(x) < 0]},
\end{flalign}
where $\mathbb{I}[c]$ is the indicator function that evaluates 
to 1 when the condition $c$ is true and to 0 otherwise. 
This equates to using the absolute values of both the \namep regularizer 
and predictive variance. 
Putting it together leads to the amortized objective,
\begin{flalign}
	\mathcal{L}_i &= \log p(y_i \g f_\theta(x_i)) + \frac{1}{n}\log p(\theta) + \frac{\lambda}{n m} \| f_\theta({\bf \hat{x}}) \|_1, \label{eq:pt_hfl_abs}
\end{flalign}
where the final term in Eq.~\ref{eq:pt_hfl_abs} is the (rescaled) sum of absolute values of 
the function applied to the $m$ evaluation inputs. 
We emphasize that this sum of absolute values 
accommodates networks with multiple outputs, 
a crucial requirement for practical use.

A stochastic optimization algorithm for 
amortized \namep is given in Algorithm~\ref{alg:pt_hfl}. 
The procedure avoids the need to explicitly calculate the Hessian 
and instead optimizes one parameter vector~$\hat{\theta}^{(\hat{r}, \lambda)}$. 
The general form of the objective optimized in 
the algorithm is Eq.~\ref{eq:pt_hfl_abs}. 
We next highlight a special case and a variant 
of Eq.~\ref{eq:pt_hfl_abs}, 
both of which admit simplified forms of~\namep.

\paragraph{Special Case: In-Sample Amortization}

When the train and test data distributions are the same, 
the data sampled to optimize the log joint probability of the model 
may also be used to regularize it, 
giving rise to the single-sample stochastic gradient, 
\begin{flalign}
	\mathcal{L}_i^{\mathrm{(IS)}} &= \log p(y_i \g f_\theta(x_i)) + \frac{1}{n} \log p(\theta) + \frac{\lambda}{nm} |f_\theta(x_i)|.
	\label{eq:in_sample_pt}
\end{flalign}
Alternatively, a data augmentation approach is applicable 
to in-sample training when the likelihood is Gaussian (see Appendix~B).

\paragraph{Variant: Parameter Uncertainty Quantification}

The same amortized regularization strategy also enables 
parameter uncertainty quantification 
by finding the regularizer $\hat{t}$ 
that minimizes the average squared L2 norm 
over the $K$ dimensions of parameters,
\begin{flalign}
	\hat{t} = \arg_{t} \min \sum_{k=1}^K \| \hat{\theta}^{(\theta_k, \lambda)} - \hat{\theta}^{(t, \lambda)}  \|_2^2. \label{eq:pt_theta}
\end{flalign}
Theorem~\ref{thrm:pt} applies to Eq.~\ref{eq:pt_theta} 
in combination with absolute value regularization to 
provide the amortized objective, 
\begin{flalign}
	\mathcal{L}_i^{\mathrm{(PU)}} &= \log p(y_i \g f_\theta(x_i)) + \frac{1}{n}\log p(\theta) + \frac{\lambda}{n m} \| \theta \|_1, \label{eq:pt_hfl_theta} 
\end{flalign}
which corresponds to a set of sparsity-inducing Laplace priors over the weights. 
Thus, when used in concert with a Laplace 
log prior $\log p(\theta)$, 
the regularization term for \namep is absorbed, 
with $\lambda$ perturbing the precision parameter.
In this case, it is particularly convenient 
to implement \namep for parameter uncertainty, 
as one needs only to adjust the strength of the existing regularizer, 
as detailed in Appendix~\ref{app:param_uc}.

\begin{table*}[t!]
\caption{Held-Out Predictive Errors with GPT-3, XL GPT-3, and ViT Architectures} \label{combined-performance-table}
\label{tab:results}
\begin{center}
\footnotesize 
\begin{tabular}{lccccccc}
\toprule 
\textbf{Method} & \multicolumn{2}{c}{\textbf{GPT-3 (125M Param.)}} & \multicolumn{2}{c}{\textbf{XL GPT-3 (1.3B Param.)}} & \multicolumn{2}{c}{\textbf{ViT (86M Param.)}} \\
\cmidrule(lr){2-3} \cmidrule(lr){4-5} \cmidrule(lr){6-7}
 & \textbf{Wikitext} & \textbf{IMDB} & \textbf{Wikitext} & \textbf{IMDB} & \multicolumn{2}{c}{\textbf{Places365}} \\
\midrule 
\multicolumn{7}{c}{\parbox{0.5\textwidth}{\centering \textbf{Negative Log Likelihood} \\ ($\downarrow$ is better) $\pm 95\%$ Confidence Interval}} \\
\midrule 
Pre-Trained     & $9.781 \pm 0.009$ & $10.374 \pm 0.012$ & $8.228 \pm 0.013$ & $9.886 \pm 0.015$ & \multicolumn{2}{c}{$5.900 \pm 0.002$}\\
MAP             & $1.383 \pm 0.027$ & $9.681 \pm 0.015$ & $1.523 \pm 0.030$ & $9.206 \pm 0.020$ & \multicolumn{2}{c}{$5.871 \pm 0.012$}\\
Ensemble      & $3.925 \pm 0.058$ & $22.84 \pm 0.005$ & $\text{---}$ & $\text{---}$  & \multicolumn{2}{c}{$5.885 \pm 0.010$}\\
MC Dropout      & $1.410 \pm 0.026$ & $9.551 \pm 0.013$ & $1.495 \pm 0.029$ & $9.137 \pm 0.019$ & \multicolumn{2}{c}{$5.878 \pm 0.006$}\\
Laplace Diag    & $1.481 \pm 0.021$ & $9.090 \pm 0.011$ & $1.273 \pm 0.021$ & $8.797 \pm 0.017$ & \multicolumn{2}{c}{$5.872 \pm 0.026$}\\
Laplace Full    & $1.277 \pm 0.021$ & $8.774 \pm 0.009$ & $1.198 \pm 0.021$ & ${8.047} \pm 0.013$ & \multicolumn{2}{c}{$5.869 \pm 0.001$}\\
\namep (this paper)  & ${1.143} \pm 0.021$ & ${8.009} \pm 0.008$ & ${1.151} \pm 0.020$ & $8.138 \pm 0.007$ & \multicolumn{2}{c}{${5.853} \pm 0.014$}\\
\midrule 
\multicolumn{7}{c}{\parbox{0.5\textwidth}{\centering \textbf{Expected Calibration Error} \\ ($\downarrow$ is better) $\pm 95\%$ Confidence Interval}} \\
\midrule 
Pre-Trained        & $0.163 \pm 0.001$ & $0.239 \pm 0.001$ & $0.233 \pm 0.002$ & $0.287 \pm 0.002$ & \multicolumn{2}{c}{$0.001 \pm 0.000$}\\
MAP             & $0.052 \pm 0.001$ & $0.224 \pm 0.002$ & $0.060 \pm 0.001$ & $0.240 \pm 0.002$ & \multicolumn{2}{c}{$0.009 \pm 0.001$}\\
Ensemble      & $0.153 \pm 0.002$ & $0.747 \pm 0.004$ & $\text{---}$ & $\text{---}$ & \multicolumn{2}{c}{$0.002 \pm 0.000$}\\
MC Dropout      & $0.036 \pm 0.000$ & $0.176 \pm 0.002$ & $0.054 \pm 0.001$ & $0.219 \pm 0.002$ & \multicolumn{2}{c}{$0.005 \pm 0.000$}\\
Laplace Diag    & $0.180 \pm 0.002$ & $0.170 \pm 0.001$ & $0.072 \pm 0.001$ & $0.203 \pm 0.001$ & \multicolumn{2}{c}{$0.001 \pm 0.000$}\\
Laplace Full    & $0.082 \pm 0.001$ & $0.136 \pm 0.001$ & ${0.024 \pm 0.000}$ & $0.126 \pm 0.001$ & \multicolumn{2}{c}{${0.000 \pm 0.000}$}\\
\namep (this paper) & ${0.019 \pm 0.000}$ & ${0.044 \pm 0.001}$ & $0.047 \pm 0.001$ & ${0.024 \pm 0.001}$ & \multicolumn{2}{c}{$0.001 \pm 0.000$}\\
\bottomrule 
\end{tabular}
\end{center}
\end{table*}

\begin{figure*}[t!]
\vskip 0.2in
\begin{center}
\centerline{\includegraphics[width=2.1\columnwidth]{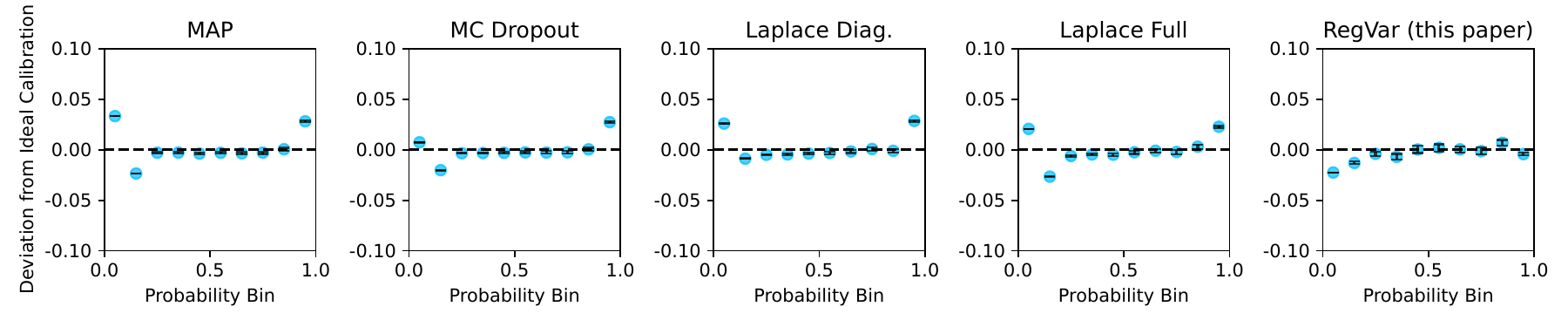}}
\caption{Calibration curves for the methods on the IMDB held-out evaluation with the XL GPT-3 network. Ideal calibration is plotted with dashed lines. 95\% confidence intervals are shown with horizontal markers.}
\label{fig:ece}
\end{center}
\end{figure*}

\section{EXPERIMENTS}
\label{sec:experiments}

We investigate the empirical performance of \namep in 
quantifying uncertainty for large models. 
The goal is to validate the theoretical developments 
from Section~\ref{sec:method} and compare 
to baselines in Bayesian deep learning. 
We next describe the experimental setup. 

\paragraph{Experimental Setup}

The experiments use a Small GPT-3 Architecture language model 
with 125~million parameters and a Vision Transformer (Base) with 86~million parameters. 
To study how the methods scale, we 
also experiment with a larger XL GPT-3 Architecture comprising 1.3 billion parameters. 
For the language models, we load the pre-trained model and fine-tune it to {\bf Wikitext} (wikitext-103-raw) comprising unprocessed text extracted from Wikipedia \citep{merity2016pointer} 
and {\bf IMDB} a binary sentiment classification dataset comprising movie reviews split evenly between positive and negative sentiments~\citep{maas2011learning}. 
For the vision model, we fine-tune the pre-trained model to 
{\bf Places365} (places365), an image dataset of 365 scenes \citep{zhou2017places}.
The fine-tuning language datasets were subsampled to 40k documents split into 80\% train, 10\% validation, and 10\% test, 
while the vision dataset was subsampled to 4k images split into 50\% train, 25\% validation, and 25\%~test. 

The language models were fine-tuned over 1~epoch with stochastic gradient descent using adaptive learning rates,\footnote{Adam with learning rate $2\times10^{-5}$ \citep{kingma2017adam}.} weight decay $0.01$,\footnote{Equivalent to $\mathcal{N}(0,50 I)$ prior on network parameters.} and batch size 4 with 4 gradient accumulation steps. These experiments were conducted on an instance configured with one NVIDIA A100 GPU with 40 GB memory, 12 CPU cores, and 143 GB of system memory. 
The vision model was fine-tuned over 5~epochs with the same settings without the gradient accumulation on an instance configured with 4 NVIDIA A10 GPUs each with 24 GB memory, 96 CPU cores, and 380 GB of system memory. 

\begin{figure}[t!]
\vskip 0.2in
\begin{center}
\centerline{\includegraphics[width=0.9\columnwidth]{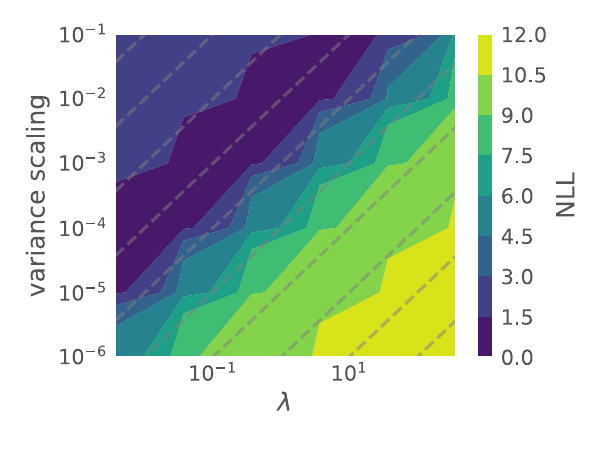}}
\caption{Contour plot of held-out negative log likelihood ($\downarrow$~is~better) for Wikitext validation 
as function of the $\lambda$ used to train \namep and variance rescaling used during inference. Lines isometric to the identity line are shown as gray dashes.}
\label{fig:lambda_contour}
\end{center}
\end{figure}

We compared a number of uncertainty quantification methods,
\begin{itemize}
	\item {\bf Maximum \emph{a posteriori} (MAP)} optimizes the log joint distribution w.r.t. the network parameters (see Eq.~\ref{eq:dnn1}). 
	\item {\bf Ensemble Last-Layer} performs MAP estimation with multiple heads at the last layer with a shared network before the last layer \citep{lakshminarayanan2017simple, osband2015bootstrapped}. During training, heads are randomly initialized and, for each training example, a head is sampled at random to receive an update. During inference, the mean and variance of the population of all head evaluations is used for uncertainty. Memory constraints prohibit approximations beyond the last layer and prevent applying ensembling to the XL GPT-3 architecture.
	\item {\bf Monte Carlo Dropout (MC Dropout)} uses dropout to generate approximate samples of the network during inference time \citep{gal2016dropout}. Since dropout is applied at every layer, it is the only other method (apart from \name) that considers uncertainty beyond the last layer. 
	\item {\bf Last-Layer Laplace with Diagonal Covariance (Laplace Diag.)} using the inverse of the diagonal of the GGN approximation of the Hessian of the parameters in the last layer of the network. 
	\item {\bf Last-Layer Laplace with Full Covariance (Laplace Full)} using the inverse of the GGN approximation of the Hessian of the parameters in the last layer of the network. 
	\item {\bf Regularization Variation (\namep)} uses Algorithm~\ref{alg:pt_hfl}. 
\end{itemize}

Further variants of the Laplace approximation can be brought to bear on large models \citep{daxberger2021laplace}. 
Last-layer diagonal/full covariance Laplace are used here 
because we found them, in line with previous studies, 
to be representative of the performance 
of different Laplace approaches. 

\paragraph{Hyperparameters}
All network weights are regularized 
by a Gaussian prior with zero mean and low precision (0.02) during fine-tuning. 
Due to computational constraints, 10 samples are used in the ensemble and MC dropout methods. 
Dropout rate, prior conditioning matrix for Laplace, 
and the $\lambda$ weight in \namep  
are tuned to maximize likelihood on validation data. 
For all methods there is a trade-off between 
accuracy and calibration; 
to control for the confounding 
influence of this trade-off on the metrics, 
we rescale the predicted variance 
by a scalar hyperparameter that is chosen for each method 
to maximize likelihood 
on the validation datasets.

\paragraph{Evaluation}

All methods are evaluated 
by held-out negative log likelihood (NLL) 
and expected calibration error (ECE). 
NLL penalizes predictive uncertainties that 
place more probability on the wrong classes, 
indicating both accuracy and calibration. 
ECE measures calibration directly using 
the weighted average error of binned predictions, 
\newcommand{\ECE}{\text{ECE}}
\newcommand{\binb}{B_m}
\newcommand{\acc}{\text{acc}}
\newcommand{\conf}{\text{conf}}
\begin{flalign}
    \ECE = \sum_{m=1}^{M} \frac{|\binb|}{n} \left| \acc(\binb) - \conf(\binb) \right|,
\end{flalign}
where $M$ is the total number of bins, $B_m$ is the set of samples whose predicted confidence scores fall into the $m$-th bin, $n$ is the total number of samples, $\text{acc}(B_m)$ is the accuracy of the samples in the $m$-th bin, and $\text{conf}(B_m)$ is the average predicted confidence of the samples in the $m$-th bin. We use $M=10$ throughout. 
Confidence intervals are calculated using $\pm 1.96$ standard errors (Wald confidence intervals).

For every method, the first step 
to evaluating the predictive distribution over classes 
is to estimate the logit means $f_\theta(x_i)$ and variances $\hat{\sigma}^2_i$ (rescaled by the aforementioned hyperparameter). 
Uncertainty in the logits is then used in the extended probit approximation \citep{gibbs1998bayesian} 
which introduces a logit correction vector $\kappa$ 
to account for model variance. 
The probit approximation correction is given by,
\begin{flalign}
    \kappa_i = \frac{1}{\sqrt{1 + \frac{\pi \hat{\sigma}^2_i}{8} }}. 
\end{flalign}
The adjusted logits are then used to form predicted probabilities 
over the clases $p(y_i \g \dat, x_i) \propto \exp \{ \kappa_i f_\theta(x_i) \}$.

\begin{figure}[t!]
\vskip 0.2in
\begin{center}
\hspace*{-10mm}  
\centerline{\includegraphics[width=1.05\columnwidth]{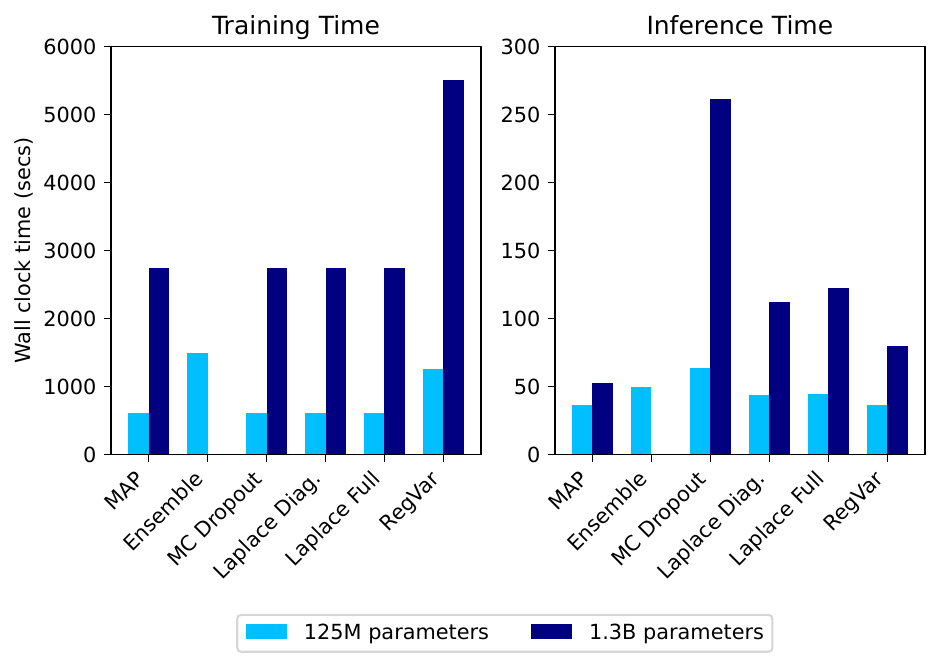}}
\caption{Fine-tune training and inference times for the methods on IMDB dataset. N.B. ensemble was prohibited in the larger model due to memory constraints.}
\label{fig:wall_clock}
\end{center}
\end{figure}

The implementation was done in PyTorch~\citep{paszke2017automatic}
using HuggingFace\footnote{\url{huggingface.co}} to load pre-trained models from \texttt{EleutherAI/\{gpt-neo-125M,gpt-neo-1.3B\}} repositories~\citep{gpt-neo} originally trained on a 800GB text dataset~\citep{gao2020pile} 
and \texttt{google/vit-base-patch16-224-in21k} repositories~\citep{wu2020visual} originally trained on ImageNet-21k~\citep{deng2009imagenet}. 
The code to run the experiments can be accessed publicly on GitHub.\footnote{\url{https://github.com/jamesmcinerney/regvar}}

\paragraph{Results}

A comparison of evaluation metrics for the methods is shown in Table~\ref{tab:results}. 
We find that \namep offers significant advantages over 
MAP across datasets 
and improves or maintains uncertainty quantification 
in comparison to the other methods. 
While Laplace Full (Last Layer) performs robustly across settings, 
improvements with \namep highlight 
the potential benefits of deep uncertainty, 
i.e., including the sources of variance from the entire network, 
which is not tractable to target with explicit Hessians at this scale.  
While NLL measures a combination of accuracy and calibration, 
ECE isolates the contribution of calibration and allows us to explore 
how this breaks down further with calibration curves in Figure~\ref{fig:ece}. 
Overall calibration is good, as evidenced by the small deviation from the ideal line, 
however most methods tend to over-represent extreme predictions (at 0/1) 
while \namep is able to calibrate the predictions away from the extremes. 

We explore the empirical behavior of the $\lambda$~hyperparameter in \namep 
by running fine-tuning with a wide range of different values. Figure~\ref{fig:lambda_contour} 
shows a contour plot of the effect of $\lambda$ (x-axis) and 
the variance rescaling hyperparameter (y-axis), the latter a hyperparameter used in all methods. 
The isometry of the contours suggests 
insensitivity of held-out negative log likelihood 
to the choice of $\lambda$ during fine-tuning. 
The two factors appear to cancel each other out and leave a dataset-dependent constant
rescaling term (approx. $\frac{1}{100}$ for Wikitext) that does not depend on $\lambda$. 

A wall clock time comparison for fine-tuning and inference 
is shown in Figure~\ref{fig:wall_clock}. 
While inference performance of \namep is advantageous 
to all other methods apart from MAP, 
obtaining the regularized network in \namep 
doubles the MAP fine-tuning time, 
a cost that is amenable to further optimization (e.g. warm starting from MAP) 
and parallelization in a similar fashion to ensembling and dropout. 

Additional experiments are detailed in Appendix~C, exploring further settings including using the \namep framework for parameter uncertainty.

\paragraph{Limitations} 
\namep is inherently a local method centered 
around the optimum of the joint distribution 
and the curvature is evaluated at one location (ignoring other modes). 
Early stopping during fine-tuning or pre-training also means 
that the chosen network parameters may not be around the 
mode of the explicit loss. 
The finite differences required by \namep 
are an additional source of error making it important to select $\lambda$ 
small enough, but not so small that underflow errors arise or, more practically, that the noise in the gradients overwhelm the signal from the regularizer. 

\section{CONCLUSIONS \& FUTURE WORK}
\label{sec:conclusions}

In this paper, we explored a new concept in Bayesian deep learning 
that identifies \emph{variation due to regularization} as 
an approximation to the predictive variance. 
We considered this phenomenon from a linearized Laplace perspective 
and evaluated it in large models. 
There is further work to explore in what cases 
it is necessary to use the uncertainty of the whole network for predictions 
to help contextualize the value of \name.

\section*{ACKNOWLEDGEMENTS}
We thank the anonymous AISTATS reviewers for their insightful feedback and suggestions. 

\bibliography{laplace}
\bibliographystyle{ACM-Reference-Format}

\newpage 
\onecolumn
\appendix

\section{PROOFS}
\label{app:proofs}
\label{app:hfl_proof}

We detail here the proofs of the theorems given in the main text. Table~\ref{tab:notation} summarizes the notation used.

{\bf Theorem 4.1} \emph{
The derivative of the prediction under the prediction-regularized MAP w.r.t. $\lambda$ recovers the variance-covariance term in Eq.~7, 
\begin{flalign}
	\nabla_{\lambda} f_{\thetahat^{(f(x),\lambda)}}(x) \at{\lambda=0} = \nabla_\theta f_{\thetahat}(x)^\top P^{-1} \nabla_\theta f_{\thetahat}(x), \label{eq:hfl1_app}
\end{flalign}
assuming $\thetahat^{(f(x),\lambda)}$ and $P^{-1}$ exist, and that ${\objective}^{(x,\lambda)}$ is continuously differentiable w.r.t. $\theta$. 
}

\begin{proof} 
To improve readability we use $(x, \lambda)$ in this proof to refer to fact that evaluation $x$ is passed to the regularizing term (i.e. always $f(x)$). 
Consider the derivative w.r.t. $\theta$ of the prediction-regularized objective at the optimum,
\begin{flalign}
	&\left. \nabla_\theta {\objective}^{(x,\lambda)}\right|_{\theta = \hat{\theta}^{(x,\lambda)}}\nnnl
	&= \nabla_\theta \left. \mathcal{L}_{\theta}\right|_{\theta = \hat{\theta}^{(x,\lambda)}} + \lambda \nabla_\theta \left. f_\theta(x)\right|_{\theta = \hat{\theta}^{(x,\lambda)}} \nnnl
	&= 0,  \label{eq:map_opt}
\end{flalign}
by definition in Eq.~8 
and by definition of the optimum. 
Now, differentiate Eq.~\ref{eq:map_opt} w.r.t. $\lambda$, 
applying the chain rule and noticing that $\thetalam$ is a function of $\lambda$,
\begin{flalign}
&\nabla^2_\theta \left. \mathcal{L}_{\theta}\right|_{\theta = \thetalam} \nabla_\lambda \thetalam + \nabla_\theta \left. f_\theta(x)\right|_{\theta = \hat{\theta}^{(x,\lambda)}}\nnnl 
&+ \lambda \nabla^2_\theta \left. f_\theta(x)\right|_{\theta = \hat{\theta}^{(x,\lambda)}}\nabla_\lambda \thetalam = 0.  \label{eq:map_diff_lambda}
\end{flalign}
Evaluating Eq.~\ref{eq:map_diff_lambda} at $\lambda=0$ eliminates the term involving the Hessian of the output. Since, by assumption, the inverse Hessian of the log joint exists, rearrange the terms, 
\begin{flalign}
	\nabla_\lambda \left. \thetalam\right|_{\lambda=0} = - \left(\nabla^2_\theta \left. \mathcal{L}_\theta\right|_{\theta = \hat{\theta}^{(x,0)}} \right)^{-1} \nabla_\theta \left. f_\theta(x)\right|_{\theta = \hat{\theta}^{(x,0)}}
	\label{eq:pre_hfl}
\end{flalign}
Finally, substituting Eq.~\ref{eq:pre_hfl} into the chain rule expansion of~$\nabla_\lambda \left. f_{\thetalam}(x)\right|_{\lambda=0}$ yields,
\begin{flalign}
	&\nabla_\lambda \left. f_{\thetalam}(x)\right|_{\lambda=0} \nnnl
	&= - \nabla_\theta \left. f_\theta(x)\right|_{\theta = \hat{\theta}}^\top \left(\nabla^2_\theta \left. \mathcal{L}_\theta\right|_{\theta = \hat{\theta}} \right)^{-1} \nabla_\theta \left. f_\theta(x)\right|_{\theta = \hat{\theta}}.
\end{flalign}

This recovers Eq.~10 
by substituting with the definition of~$P$~(Eq.~4). 
\end{proof}

{\bf Theorem 4.2} \emph{The optimal regularization term defined by Eq.~13 
is $\hat{r} = \frac{1}{M} \sum_{m=1}^M f(\hat{x}_m)$ 
when $\hat{\theta}^{(f(x),\lambda)}$ and $P^{-1}$ exist and ${\objective}^{(f(x),\lambda)}$ is continuously differentiable w.r.t. $\theta$. 
}

\begin{proof}
Take the derivative of the objective in Eq.~13 
w.r.t.~$r$,  
\begin{flalign}
&\nabla_r \sum_{m=1}^M \| \hat{\theta}^{(f(x_m), \lambda)} - \hat{\theta}^{(r, \lambda)}  \|_2^2 \nnnl
&=\nabla_r \sum_{m=1}^M \| (\hat{\theta}^{(f(x_m), \lambda)} - \hat{\theta}) - (\hat{\theta}^{(r, \lambda)} - \hat{\theta}))  \|_2^2 \nnnl
&= \nabla_r \sum_{m=1}^M \| P^{-1} \nabla_\theta f(\hat{x}_m) -  P^{-1} \nabla_\theta r  \|_2^2 \label{eq:pt_P}\\
&= \sum_{m=1}^M \nabla_r \| P^{-1} (\nabla_\theta f(\hat{x}_m) - \nabla_\theta r)  \|_2^2 \nnnl
&= 2(P^{-1})^\top P^{-1} ( \sum_{m=1}^M (\nabla_\theta f(x_m) - \nabla_\theta r)), \label{eq:pt_obj}
\end{flalign}
using Theorem~4.1 
to express the difference 
in optimal parameters in terms involving the inverse Hessian in Eq.~\ref{eq:pt_P}. 
Finally, set the derivative in Eq.~\ref{eq:pt_obj} to zero 
and solve for~$r$, the unique minimizer of the objective (up to a constant that does not depend on $\theta$). 
\end{proof}

\captionsetup{justification=centering}
\begin{table}[t!]
    \centering
    \caption{\textbf{Notation Summary}}
    \renewcommand{\arraystretch}{1.5}  
    \begin{tabular}{>{$}c<{$} p{14cm}} 
        \toprule
        \textbf{Symbol} & \textbf{Definition} \\
        \midrule
        \theta & Network parameters.\\
        \objective & Log joint distribution parameterized by $\theta$.\\
        \thetahat & Maximum \emph{a posteriori} parameters under the joint distribution~(Eq.~1).\\
        f_{\thetahat} & Network output with parameters $\thetahat$.\\
        \nabla_\theta f_{\thetahat} & Derivative of network output w.r.t. $\theta$ evaluated at $\theta = \thetahat$.\\
        {\objective}^{(t_\theta,\lambda)} & Log joint distribution parameterized by $\theta$ and regularized by $\lambda t_\theta$, e.g. $t_\theta$ could be a network output~$f_\theta(x)$ (Eq.~8) 
        or fine-tuning regularizer~$r$~(Eq.~13).\\ 
        \thetahat^{(t_\theta,\lambda)} & Parameters that maximize ${\objective}^{(t_\theta,\lambda)}$, e.g., see Eq.~9.\\ 
        \nabla_{\lambda} f_{ \thetahat^{(t_\theta,\lambda)}} & Derivative of network output under parameters $\thetahat^{(t_\theta,\lambda)}$ w.r.t. regularization weight~$\lambda$. Note that the optimal network parameters themselves are a function of $\lambda$, making the overall term a non-linear expression involving $\nabla_{\lambda} f_{\thetahat}$ (see Eq.~21).\\
        \bottomrule
    \end{tabular}
        \label{tab:notation} 
\end{table}

\section{ALGORITHMS FOR REGULARIZATION VARIATION (\name)}
\label{sec:hfl_algs}
\label{app:data_aug}

In this section we expand on variations of the \namep strategy 
and provide relevant algorithms where applicable.

\subsection{Pointwise \name}

The algorithm that follows most directly 
from Theorem~1 
targets the uncertainty calculation for an arbitrary evaluation point $x$ 
(not necessarily from the training or test set) 
and is given in Algorithm~\ref{alg:hfl}. 
The computational burden of recalculating the optimal regularized parameters 
for every different evaluation input 
means that this approach is not scalable to large models, 
motivating extensions to the regularization term 
which we consider next. 

\begin{algorithm*}[h]
   \caption{Pointwise \name}
   \label{alg:hfl}
\begin{algorithmic}
   \STATE {\bfseries Input:} network $f_\theta$, fitted MAP $\hat{\theta}$, training data $\D$ of size $n$, evaluation input $x$, scalar $\lambda$, step size schedule $\gamma(\cdot)$
   \STATE Warm start $\hat{\theta}^{(x, \lambda)} \gets \hat{\theta}$
   \STATE Initialize step counter $j=0$ 
   \REPEAT
   \STATE Sample training example $(x_i,y_i) \sim \D$
   \STATE Follow stochastic gradient $\hat{\theta}^{(x, \lambda)} \gets \hat{\theta}^{(x, \lambda)} + N \gamma(j) \nabla_\theta \tilde{\mathcal{L}}_{\theta}^{(x,\lambda)}(x_i,y_i)$
   \STATE $j \gets j + 1$
   \UNTIL{$\hat{\theta}^{(x, \lambda)}$ converges}
   \STATE {\bfseries Return:} predictive variance $\hat{\sigma}_f^2(x) := \frac{1}{\lambda} | f_{\hat{\theta}^{(x, \lambda)}}(x) - f_{\hat{\theta}}(x) |$ 
\end{algorithmic}
\end{algorithm*}

\subsection{Data-Augmentation Objective}

When the likelihood is Gaussian, 
the \namep regularization term can be absorbed into the likelihood. 
In particular, completing the square in Eq.~16, by pulling the term $\frac{\lambda \sigma^2}{n M}(\frac{\lambda \sigma^2}{2n M} - y_i)$ from a constant that does not depend on $\theta$, yields the objective, 
\begin{flalign}
	\mathcal{L}_i^{\mathrm{(DA)}} &= \log \mathcal{N}\left(y_i + \frac{\tilde{\lambda}}{n M} \sigma^2 \;\bigr|\; f_\theta(x_i), \sigma^2\right) + \frac{1}{n}\log p(\theta),
	\label{eq:aug_pt}
\end{flalign}
for observation variance hyperparameter $\sigma^2$. 
Under the common Gaussian likelihood assumption, 
Eq.~\ref{eq:aug_pt} implies that 
pre-trained \namep is performed 
by data augmentation, 
without requiring \emph{any} adjustment to the model architecture or training procedure.

\begin{algorithm*}[t!]
   \caption{Amortized \namep for Parameter Uncertainty}
   \label{alg:pt_hfl_pu}
\begin{algorithmic}
   \STATE {\bfseries Input:} network $f_\theta$, fitted MAP $\hat{\theta}$, training data $\D$ of size $n$, evaluation inputs $\{ \hat{x}_{i=1}^m \} $, scalar $\lambda$, step size schedule $\gamma(\cdot)$
   \STATE Warm start $\hat{\theta}^{(\hat{r}, \lambda)} \gets \hat{\theta}$
   \STATE Initialize step counter $j=0$
   \REPEAT
   \STATE Sample training example $(x_i,y_i) \sim \D$
   \STATE Follow gradient $\hat{\theta}^{(\hat{r}, \lambda)} \gets \hat{\theta}^{(\hat{r}, \lambda)} + n \gamma(j) \nabla_\theta \mathcal{L}_i^{\mathrm{(PU)}}\;\;\;$ \act{See Eq.~18} 
   \STATE $j \gets j + 1$
   \UNTIL{$\hat{\theta}^{(\hat{r}, \lambda)}$ converges}
   \STATE {\bfseries Return:} parameter variance $\hat{\sigma}_{\theta}^2(x) := \frac{1}{\lambda} | \hat{\theta}^{(\hat{r}, \lambda)} - \hat{\theta} |$
\end{algorithmic}
\end{algorithm*}

\subsection{\namep for Parameter Uncertainty}
\label{app:param_uc}

For completeness, we present Algorithm~\ref{alg:pt_hfl_pu} 
that corresponds to the parameter uncertainty objective 
given in Eq.~\ref{eq:pt_hfl_theta} in the main text. 
We apply this algorithm in Section~\ref{app:pu}. 

\section{FURTHER EXPERIMENTS AND HYPERPARAMETER SELECTION}
\label{app:further_exp}

In this section, we evaluate \namep 
and compare against the full Hessian and its approximations.

\subsection{Synthetic Data}
\label{sec:exp_synth}

We explore four data generating distributions designed to test both in-distribution inference 
and extrapolation to ``in-between'' uncertainty~\citep{foong2019between}. 
A feedforward neural network is used with 50 hidden units and $\tanh$ activations to ensure that the Hessian exists, and optimized using Adam \citep{kingma2017adam} with a learning rate of 0.005. 
Four methods for quantifying the epistemic uncertainty of the network are compared,
\begin{itemize}
	\item {\bf Exact Hessian} using Eq.~4 
	and both terms in the likelihood Hessian (Eq.~5). 
	\item {\bf GGN} using Eq.~4 
	with only the second term in likelihood Hessian (Eq.~5). 
	\item {\bf Eigenvector approximation} of the GGN for $P^{-1} \approx A \Lambda^{-1} A^\top$ using the top $k$ eigenvectors $A$ of $P$ (with corresponding eigenvalues $\Lambda$) where $k = \log_e(\mathrm{dim}(\theta))$ rounded to the nearest positive integer. The GGN is chosen because 
it only makes sense to use the eigenvector approximation 
when the Hessian is large. 
	\item {\bf \name} (this paper), using Algorithm~1. 
\end{itemize}
The metrics used to evaluate uncertainty performance are,
\begin{itemize}
	\item {\bf Prediction interval coverage probability (PICP)} assesses calibration 
	by evaluting how well prediction intervals cover the observed values. It is calculated 
	as the proportion of times the held-out observed value $y$ falls within 
	the 95\% confidence interval of the predicted observations (\emph{closer to 0.95 is better}). 
	\item {\bf Continuous Ranked Probability Score (CRPS)} is a generalization 
	of the Brier score for continuous outcomes, measuring the difference between 
	the CDF of the predicted observations and the CDF of the observed response variable (\emph{smaller is better}). 
	\item {\bf Negative log likelihood (NLL)} evaluates the negated predicted log likelihood of the held out observations~(\emph{smaller is better}). 
\end{itemize}

\begin{figure*}[t!]
  \centering
  \subfigure[Full Hessian]{
    \includegraphics[width=0.23\textwidth]{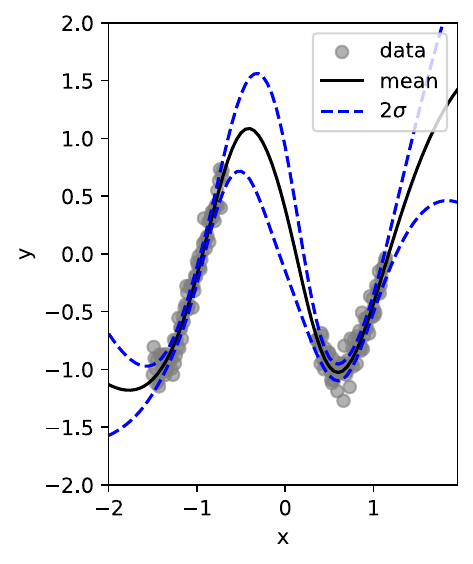}
    \label{fig:sin_full}
  }
  \subfigure[GGN]{
    \includegraphics[width=0.23\textwidth]{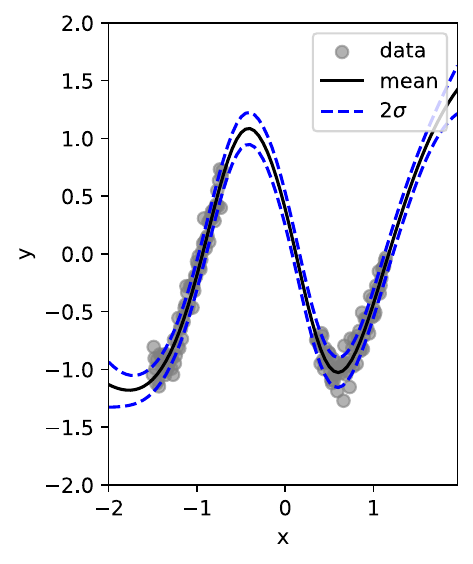}
    \label{fig:sin_ggn}
  }
  \subfigure[Eigen. Approx.]{
    \includegraphics[width=0.23\textwidth]{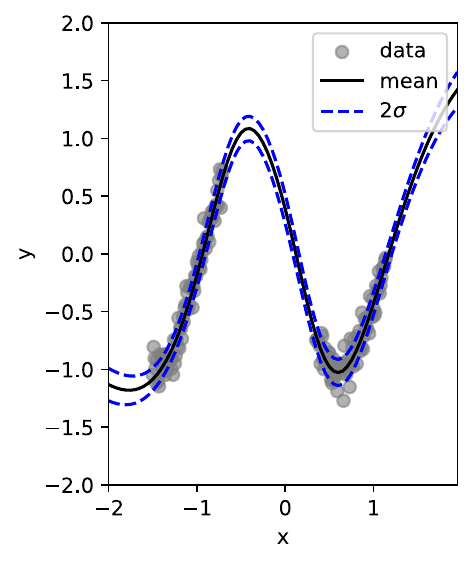}
    \label{fig:sin_diag}
  }
  \subfigure[HFL (this paper)]{
    \includegraphics[width=0.23\textwidth]{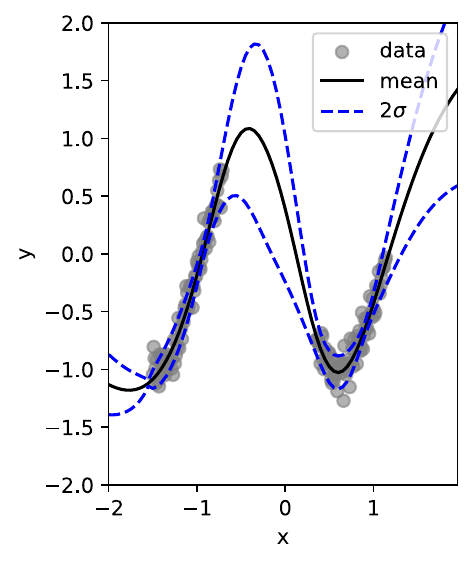}
    \label{fig:sin_hfl}
  }
  \caption{Illustration of epistemic uncertainties of the predicted mean for the Sin-Inbetween OOD task.}
  \label{fig:sin_ood}
\end{figure*}

The synthetic data-generating distributions are,
\label{app:exp_details}
\begin{itemize}
	\item {\bf Quadratic-Uniform} Draw 32 data points $x \sim \mathcal{N}(0, 1)$ and let $y = \frac{1}{10}x^2 - \frac{1}{2}x + 5 + \frac{1}{10}\epsilon$, where~$\epsilon \sim \mathcal{N}(0,1)$.  
	\item {\bf Quadratic-Inbetween} Draw 16 data points $x \sim \mathrm{Uniform}(-2, -\frac{1}{2})$ and 16 data points $x \sim \mathrm{Uniform}(\frac{4}{5}, \frac{5}{2})$, using the same response distribution as Quadratic-Uniform.
	\item {\bf Sin-Uniform} Sample 160 data points $x \sim \mathrm{Uniform}(-\frac{3}{2},\frac{23}{20})$ and let $y = -\sin(3x - \frac{3}{10}) + \frac{1}{10}\epsilon$, where $\epsilon \sim \mathcal{N}(0,1)$. 
	\item {\bf Sin-Inbetween} Fix 160 data points evenly spaced in ranges $[-1.5, -0.7)$ and $[0.35, 1.15)$ and let use the same response as Sin-Uniform. 
\end{itemize}

\begin{figure*}[t!]
  \centering
  \subfigure[Hyperparameter sweeps for observational variance.]{
    \includegraphics[width=0.48\textwidth]{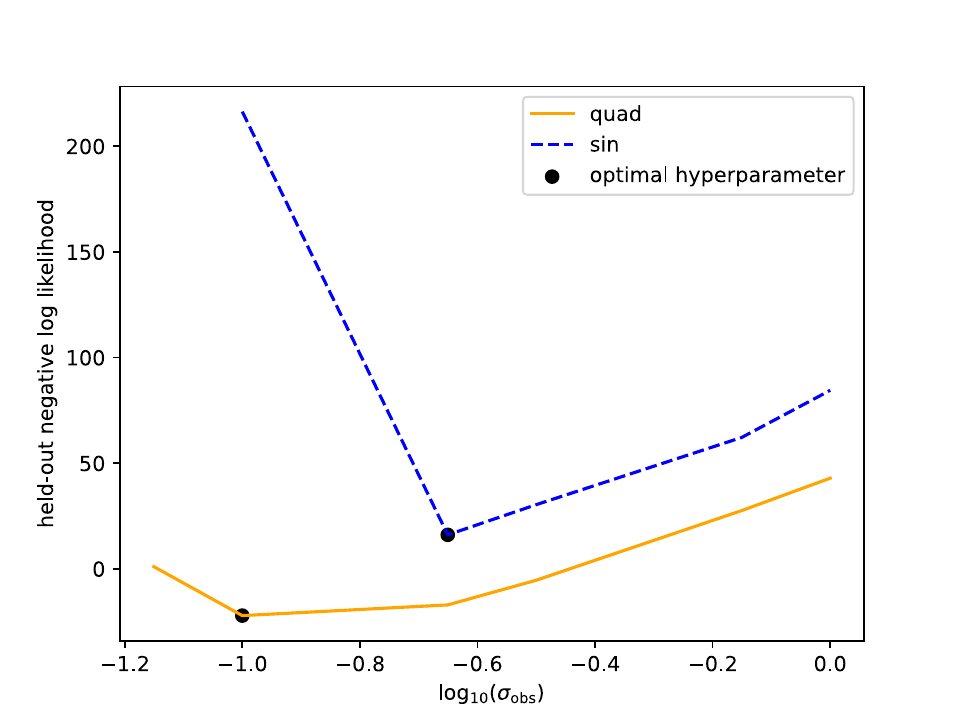}
	\label{fig:hyp_sweep}
  }
  \subfigure[Lambda sweep]{
    \includegraphics[width=0.48\textwidth]{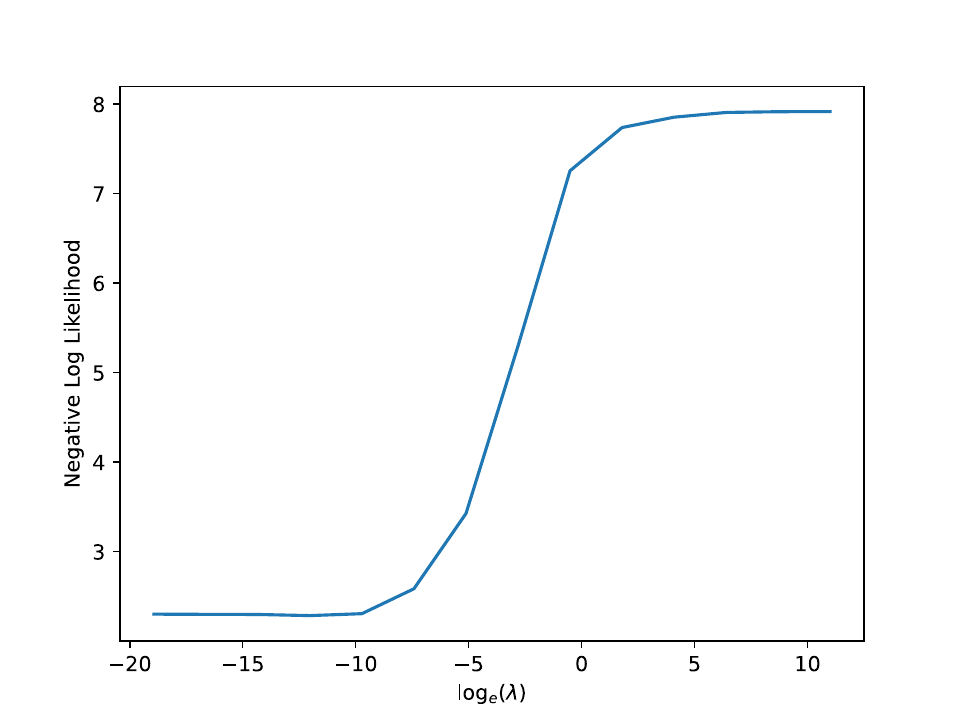}
	\label{fig:lambda_sweep}
  }
  \caption{Hyperparameter selection for modeling and \namep algorithm.}
  \label{fig:hyp_sweeps}
\end{figure*}

\label{app:evaluation}
\label{app:hyp_sweep}

\begin{table*}[h!]
\centering
\label{table:metrics}
\resizebox{\textwidth}{!}{%
\begin{tabular}{@{}lcccccccc@{}}
\toprule
& \multicolumn{8}{c}{Datasets} \\ 
\cmidrule(l){2-9} 
& \multicolumn{2}{c}{Quadratic} & \multicolumn{2}{c}{Quadratic-Inbetween} & \multicolumn{2}{c}{Sin} & \multicolumn{2}{c}{Sin-Inbetween} \\
& In-Distribution & OOD & In-Distribution & OOD & In-Distribution & OOD & In-Distribution & OOD \\ 
\midrule
\multirow{3}{*}{\emph{Full Hessian}} & \emph{PICP: 0.9375} & \emph{PICP: 0.9556} & \emph{PICP: 0.9688} & \emph{PICP: 0.8667} & \emph{PICP: 1.0000} & \emph{PICP: 0.8625} & \emph{PICP: 1.0000} & \emph{PICP: 0.9250} \\
& \emph{CRPS: 0.003156} & \emph{CRPS: 0.009017} & \emph{CRPS: 0.003797} & \emph{CRPS: 0.003331} & \emph{CRPS: 0.04666} & \emph{CRPS: 0.07498} & \emph{CRPS: 0.04486} & \emph{CRPS: 0.06022} \\
& \emph{NLL: -23.924} & \emph{NLL: -38.027} & \emph{NLL: -23.290} & \emph{NLL: -24.538} & \emph{NLL: -75.764} & \emph{NLL: 26.004} & \emph{NLL: -73.178} & \emph{NLL: 10.029} \\
\midrule
\midrule
\multirow{3}{*}{GGN} & PICP: 0.9375 & PICP: 0.9556 & PICP: {0.9688} & PICP: 0.8444 & PICP: 1.0000 & PICP: 0.8250 & PICP: 1.0000 & PICP: 0.8875 \\
& CRPS: 0.003762 & CRPS: 0.007771 & CRPS: 0.003992 & CRPS: 0.0008355 & CRPS: 0.04652 & CRPS: 0.05493 & CRPS: 0.04482 & CRPS: 0.04302 \\
& NLL: -23.790 & NLL: -38.751 & NLL: {-23.135} & NLL: -23.504 & NLL: -76.025 & NLL: 101.6 & NLL: -73.260 & NLL: 24.903 \\
\midrule
\multirow{3}{*}{Eigen. Approx.} & PICP: 0.9375 & PICP: {0.9556} & PICP: 0.9688 & PICP: 0.8444 & PICP: {1.0000} & PICP: 0.8250 & PICP: {1.0000} & PICP: 0.8875 \\
& CRPS: 0.003760 & CRPS: 0.007705 & CRPS: 0.003566 & CRPS: 0.0001922 & CRPS: {0.04605} & CRPS: {0.05364} & {CRPS: 0.04454} & CRPS: {0.04208} \\
& NLL: -23.789 & NLL: {-38.760} & NLL: -23.098 & NLL: -22.893 & NLL: {-77.027} & NLL: 109.7 & NLL: {-73.785} & NLL: 27.112 \\
\midrule
\multirow{3}{*}{\name} & PICP: {0.9375} & PICP: 0.9333 & PICP: 0.8750 & PICP: {0.8444} & PICP: 1.0000 & PICP: {0.8375} & PICP: 1.0000 & PICP: {0.9250} \\
& CRPS: {0.002347} & CRPS: {0.006976} & CRPS: {0.001782} & CRPS: {4.735e-05} & CRPS: 0.04691 & CRPS: 0.06397 & CRPS: 0.05016 & CRPS: 0.06301 \\
& NLL: {-23.899} & NLL: -38.440 & NLL: -22.571 & NLL: {-23.676} & NLL: -75.245 & NLL: {45.483} & NLL: -63.218 & NLL: {22.915} \\
\bottomrule
\end{tabular}%
}
\caption{Uncertainty quantification performance across datasets and methods. 
Full Hessian (top line) is intractable at scale and included here for reference.}
\label{tab:results}
\end{table*}

For both the quadratic and sin datasets, 
the observational variance hyperparameter 
was selected through grid search over the set $\{ 0.005, 0.01, 0.05, 0.1, 0.5, 1.0 \}$ 
minimizing negative log likelihood for a held-out validation set of equal size to the training dataset. 
The results are summarized in Figure~\ref{fig:hyp_sweep}, with the optimal hyperparameter chosen indicated by a marker. The prior variance was fixed at 3.0, which was the minimum value to the nearest 0.1 for which the Hessian was sufficiently conditioned across all tasks to allow for inversion.

\begin{figure*}[t!]
\vskip 0.2in
\centering
\begin{center}
\adjustbox{trim={1.9cm} 0 {1.8cm} 0, clip}{\includegraphics[width=1.2\columnwidth]{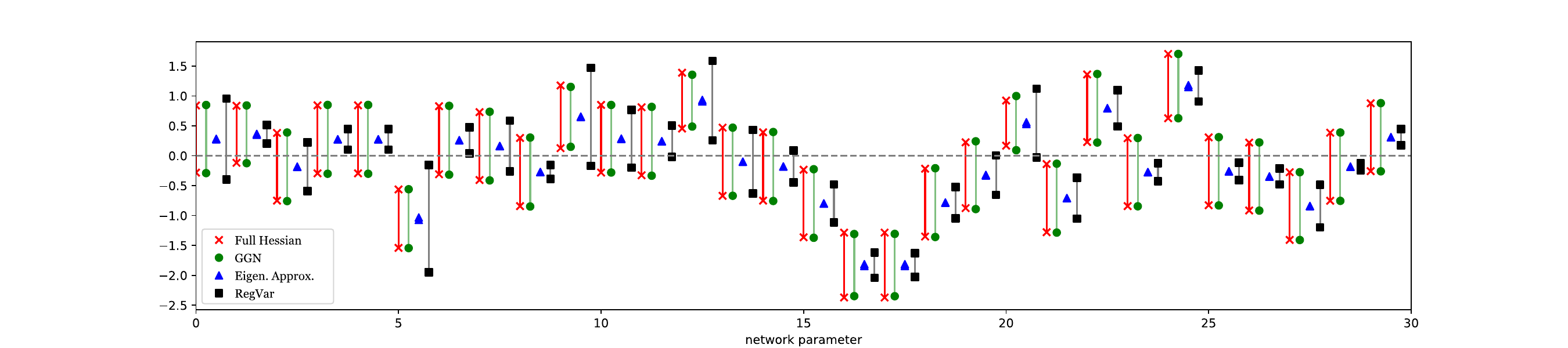}}
\caption{The uncertainty ranges of a random subset of 30 parameters of the neural network indicate that many parameters include 0 in their intervals. \namep recovers much of their structure just via additional regularization of the network parameters in the objective function.}
\label{fig:params}
\end{center}
\vskip -0.2in
\end{figure*}

\begin{figure}[h!]
\vskip 0.2in
\begin{center}
\includegraphics[width=0.60\columnwidth]{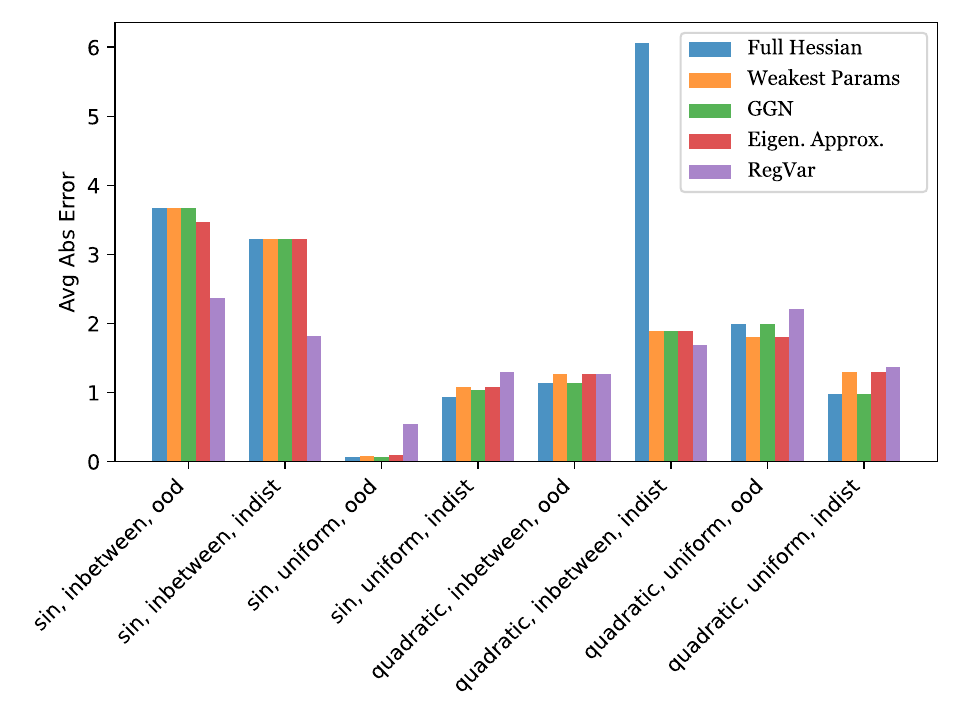}
\caption{Absolute errors in predicted means for all the tasks with networks with added sparsity using the uncertainty quantification methods.}
\label{fig:errorbars}
\end{center}
\vskip -0.2in
\end{figure}

The results of the evaluations are summarized in Table~\ref{tab:results}. 
We find that \namep performs commensurately 
with other approximations, and indeed the full Hessian, 
despite taking a significantly different approach. 
In particular, in several cases it achieves the best out-of-distribution (OOD) performance, 
likely due to its ability to avoid the GGN assumption 
that ignores an important part of the curvature around the MAP. 
To illustrate this point, 
we plot the predicted mean output 
for the methods on the Sin-Inbetween OOD task in Figure.~\ref{fig:sin_ood}. 
There is a noticeable difference to the GGN and eigenvector approximation, 
implying that in practice the residual of the loss 
may not be close to zero.

\subsection{Parameter Uncertainty Experiment}
\label{app:pu}

We also evaluated the parameter regularization 
approach given in Eq.~18. 
Figure~\ref{fig:params} shows the 
1 standard deviation interval about the predicted mean 
under the various methods. 
Recall that \namep for parameter uncertainty requires only 
to apply a small amount of additional L1 regularization to the parameters. 
Nonetheless, the intervals indicated often align with those requiring 
full Hessian inversion. 
These parameter uncertainties may be used for model introspection. 
As an example, 
we considered the task of adding sparsity to a network, 
setting parameters to 0 when their uncertainty intervals overlap with 0. 
The mean absolute errors of the resulting predictions are shown in Figure~\ref{fig:errorbars}.
\namep does competitively, particularly for in-between uncertainty. 
An outlier result is the absolute error of the full Hessian 
in the Quadratic-Inbetween in-distribution task. 
This was due to mis-selecting an entire group of units to be set to 0, 
greatly impacting the shape of predictions.

\end{document}